\documentclass[journal]{IEEEtran}

\pdfoutput=1

\newcommand{\todo}[1]{}

\newcommand{\MEC}[1]{}

\usepackage{amsmath}
\usepackage{amsfonts}
\usepackage{amsthm}
\usepackage{graphicx} 
\usepackage{algorithm}
\usepackage{algorithmic}
\usepackage{setspace}
\usepackage{color}		
\usepackage{subfigure}

\ifCLASSINFOpdf
\else
\fi
\usepackage{algorithm}
\usepackage{algorithmic}
\newcommand{\U}{\mathcal{U}}

\newcommand{\R}{\mathbb{R}}

\newcommand{\setF}{\mathcal{F}}
\newcommand{\conv}{{\rm conv}}

\newtheorem{theorem}{Theorem}
\newtheorem{lemma}{Lemma}

\newtheorem{remark}{Remark}

\begin{document}
%
\title{Mixed Strategy for \\ Constrained Stochastic Optimal Control}
%
%
%

\author{Masahiro Ono,~Mahmoud El Chamie, Marco Pavone,~and~Beh\c{c}et A\c{c}ikme\c{s}e
\thanks{Masahiro Ono ({\tt\small ono@jpl.nasa.gov}) is with Jet Propulsion Laboratory, California Institute of Technology, 4800 Oak Grove Drive, Pasadena, CA, USA. Copyright 2016 California Institute of Technology. Government sponsorship acknowledged.}
\thanks{Mahmoud El Chamie ({\tt\small melchami@uw.edu}) and Beh\c{c}et A\c{c}ikme\c{s}e ({\tt\small  behcet@uw.edu}) are with University of Washington, Seattle, WA  98195.}
\thanks{Marco Pavone ({\tt\small pavone@stanford.edu}) is with the Department of Aeronautics and Astronautics, Stanford University, Stanford, CA 94305}
}

%
%

\markboth{~}%
{Shell \MakeLowercase{\textit{et al.}}: Bare Demo of IEEEtran.cls for Journals}
%



\maketitle

\begin{abstract}
Choosing control inputs randomly can result in a reduced expected cost in optimal control problems with stochastic constraints, such as  stochastic model predictive control (SMPC). 
We consider a controller with initial randomization, meaning that the controller randomly chooses from $K$+1 control sequences at the beginning (called $K$-randimization).
It is known that, for a finite-state, finite-action Markov Decision Process (MDP) with $K$ constraints, $K$-randimization is sufficient to achieve the minimum cost.
We found that the same result holds for stochastic optimal control problems with continuous state and action spaces.
Furthermore, we show the randomization of control input can result in reduced cost when the optimization problem is nonconvex, and the cost reduction is equal to the duality gap.
We then provide the necessary and sufficient conditions for the optimality of a randomized solution, and develop an efficient solution method based on dual optimization. 
Furthermore, in a special case with $K=1$ such as a joint chance-constrained problem, the dual optimization can be solved even more efficiently by root finding.
Finally, we test the theories and demonstrate the solution method on multiple practical problems ranging from path planning to the planning of entry, descent, and landing (EDL) for future Mars missions.  
\end{abstract}


%
\IEEEpeerreviewmaketitle

\section{Introduction}

\IEEEPARstart{T}{he}  main finding of this paper is that, in optimal control problems with stochastic constraints, choosing control inputs randomly can result in a less expected cost than deterministically optimizing them.
To communicate the idea, consider the following toy problem illustrated in Figure \ref{fig:toy}. The goal is to plan a path to go to the goal with a minimum expected cost while limiting the chance of failure to 1 \%.
There are two path options, A and B. A has the expected cost of 20 and the chance of failure is 0.5\%; B has the expected cost of 10 and the chance of failure is 1.5\%. Choosing B violates the chance constraint, hence the optimal solution is A if only deterministic choice is allowed. However, we can create a mixed solution by flipping a coin to randomly choose between A and B. Assuming the probability of head and tail is 0.5, the resulting mixed solution has the expected cost of 15 and a 1\% chance of failure, which satisfies the chance constraint. The expected cost of the mixed solution is less than that of A.
In the example above, A and B are represented by two different sequences of control inputs. When a mixed strategy is employed, the system flips a coin once at the beginning, and choose a sequence according to the result of the coin flip. Once a sequence is selected, the system sticks to the sequence until the end.

\begin{figure}[tb]
  \begin{center}
    \includegraphics[width=0.8\columnwidth]{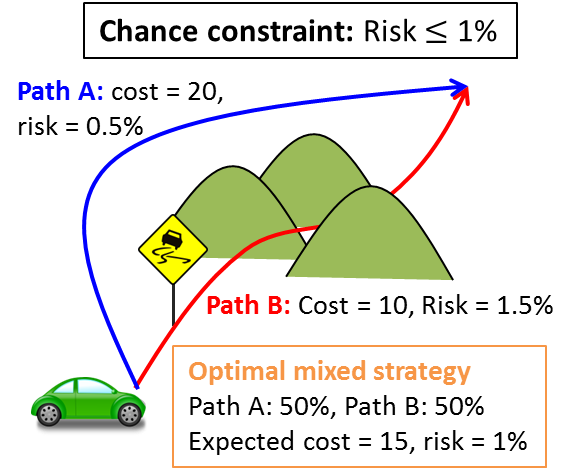}
  \end{center}
  \caption{A toy example illustrating the concept of mixed-strategy SMPC. The driver chooses between Paths A and B by a coin flip with equal probability. If the  upper bound on risk is 1\%, the mixed strategy satisfies the chance constraint and the expected cost is less than the optimal deterministic choice, Path A. } 
  \label{fig:toy}
    \begin{center}
    \includegraphics[width=0.8\columnwidth]{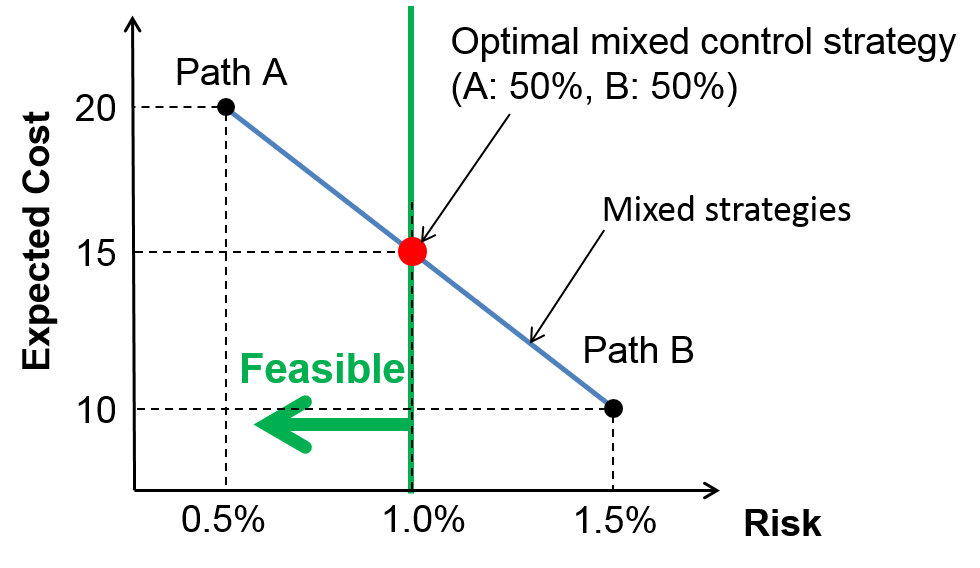}
  \end{center}
  \caption{Solutions of Figure \ref{fig:toy} in cost-risk space. Mixed strategies convexify the solution set.} 
  \label{fig:toy2}
\end{figure}

Mixed strategy is essentially a convexification. As shown in Figure \ref{fig:toy2}, in the cost-risk space, the set of the pure strategies is a nonconvex set consisting of two points. 
The solution set of mixed strategies is a line segment between A and B, which is the convex hull of the pure solution set.
In general, as shown in Figure \ref{fig:duality}, when the original problem is nonconvex, introducing mixed strategy extends the solution set, which could improve the cost of the optimal solution. 
The improvement is equal to the duality gap.
When there is no duality gap, the optimal mixed solution is equivalent to the optimal pure solution (i.e., choosing the optimal pure solution with the probability of one.)
Therefore, the optimal solution to the mixed strategy problem is always as good as the optimal solution to the original problem.

In general, unlike the illustrative example above, a stochastic optimal control problem has infinitely many solutions. 
An optimal mixed-strategy controller first computes a finite number of control sequences, them randomly chooses one from them.
The formal problem definition is given as an extension to a standard finite-horizon, constrained stochastic optimal control problem, where $K$+1 control sequences \textit{and} the probability to choose the control sequences are optimized (called $K$-randomization).
The degree of randomization, $K$, is pre-specified.
The controller chooses one control sequence at the beginning (hence called \textit{initial randomization}).
Then the chosen control sequence is executed till the end.

The contribution of this paper is three-fold.
First, we provide rigorous characterizations of mixed-strategy, constrained stochastic optimal control, which are summarized in Theorems 1-3 in Section \ref{sec:method}.
Theorem 1 in Section \ref{sec:degree_of_randomization} provides the sufficient degree of randomization. 
Specifically, for a problem with $K$ constraints, $K$-randomization is sufficient for optimality.
Theorem 2 in Section \ref{sec:cost_reduction} states that the attainable cost reduction by mixed strategy is equivalent to the duality gap in the original (non-randomized) problem. This is because the original problem is convexified by mixed strategy. In other words, mixed strategy can improve the solution only if the original problem is nonconvex.
Theorem 3 in Section \ref{sec:solution} provides the necessary and sufficient conditions for optimality, which is built upon the consists of a subset of KKT conditions and an additional constraint requiring that all the candidate control sequences are the minimizers of the Lagrangian function with the optimal dual variable.

The second contribution of this paper is to develop an efficient solution approach to the mixed-strategy, constrained stochastic optimal control problem.
A naive solution approach requires co-optimization of $K+1$ control sequences as well as the probability distribution, which is significantly more complex than the original optimization problem without randomization.
A key observation is that, since the mixed-strategy optimization problem is the convexification of the original problem, their dual optimal solutions are the same.
This observation leads to our general solution method that solves the dual of the \textit{original} optimization problem (without randomization).
The primal optimal solution for the mixed-strategy problem can be recovered from the dual optimal solution.
Furthermore, in a special case with $K=1$, we provide an even more efficient solution approach that solves the dual problem with root finding.
For a more specific case where the proposed approach is applied to a linear SMPC with nonconvex constraints, we present an efficient and approximate solution approach where the minimization of Lagrangian function is approximated by a MILP through piece-wise linearization of the cumulative distribution function of the state uncertainty.

The third contribution is to validate the theories and demonstrate the solution method in various practical scenarios. 
We first show an example of SMPC-based path planning with obstacles and a joint chance constraint, and show that mixed strategy indeed improves the expected cost.
We also show that the proposed dual solution approach is also applicable to a finite-state optimal control problem. 
Two examples are presented in this domain: path planning with obstacles, and the planning of entry, descent, and landing (EDL) for future Mars rover/lander missions.

\subsection{Related Work}

In game theory, mixed strategy is usually discussed in a context of simultaneous adversary game. 
A classical example is paper-rock-scissors, where the sole Nash equilibrium is to uniformly randomize the strategy for both players.
An underlying assumption here is that both players optimize their strategy given the strategy of the other player. 
The SMPC problem is different in that one player (controller) optimizes her strategy given the strategy of the other (the nature) but not \textit{vice versa}.
In other words, one player is cognitive while the other is blind.
In paper-rick-scissors with cognitive and blind players, the cognitive player cannot be better off by employing a mixed strategy.   
Therefore, the fact that the cognitive player \textit{can} be better off with a mixed strategy in optimal control is seemingly contradictory. 
This is because, unlike players in classical game-theoretic settings, the controller solves a constrained optimization. 
Intuitively, the constraint and the objective work adversarially, like two players within a controller.

It is known that mixed strategy can improve the solution of constrained Markov decision process (MDP). 
Major results on this subject, including randomization, are summarized in \cite{Altman1999_ConstrainedMDP}.
A stochastic optimal control problem can be viewed as an MDP with \textit{continuous} state and control spaces.
 
 \MEC{ 2 paragraphs for the comparison of related work on MDP.}
 The majority of the existing methods for solving constrained MDPs use the idea of \emph{convex-analytical} (CA) approach \cite{Borkar2002}.
The CA approach optimizes the performance metric of the MDP by reducing the problem to an optimization of a linear function over the set of occupancy measures; hence, formulated as a linear program.  The CA has shown to be useful for solving MDPs with multiple criteria and constraints when the constraints have the same additive structure as the performance measure (i.e., linear function over the set of occupancy measures). This paper on the other hand, do not limit the constraints/performance metric to any structure, thus it can be used for solving a more general class of constrained MDPs. 

The scope of our work considers the performance of a mixture of nonrandomized policies (i.e., only considers initial randomization).  It is worth mentioning that previous work has studied whether it is possible to \emph{split} a randomized policy (i.e., randomization of feedback control law) into a mixture of deterministic policies while preserving performance \cite{Feinberg:2012}. It is shown that any Markov policy is a mixture of nonrandomized Markov policies \cite[Theorem 5.2]{Feinberg:1996}. This inclusion suggests that this work also generalizes to randomization of control policy.

Note that the proposed method is fundamentally different from randomized SMPC methods such as scenario-based MPC \cite{BernardiniBemporadCDC09,CalafioreScenarioReview}. In scenario-based MPC, the optimal control inputs are deterministic but the solution method to obtain them is randomized. In contrast, in this work, the optimal control inputs are randomized  but the solution method is deterministic.
Randomized control input was considered in a control theoretical context by \cite{Behcet_MCMC,Behcet_MCMC_safety}.
The problem considered in these studies is the probabilistic coordination of swarms of autonomous agents using a Markov chain controller. 
Here randomized control is used for a different purpose than in our work. 
In the Markov chain control randomized control inputs are used to achieve the desired spacial density distribution of the swarm agents without assuming inter-agent communication. In contrast, in our work, randomized control inputs are used to achieve less expected cost.

\todo{mixed strategy chance constrained programming?}

\todo{interpretation - multi robot, schodinger cat}

\section{Method}
\label{sec:method}

The following is the rough sketch of the proposed solution process.
\begin{enumerate}
\item Prove that the original problem and the mixed strategy problem share the same dual optimal solution
\item Compute the dual optimal solution by solving the dual of the original problem
\item Recover the primal optimal solution of the mixed strategy problem from the dual optimal solution
\end{enumerate}
The solution method is explained in detail in the following subsections.

\subsection{Problem Formulation}

We first formulate a pure-strategy problem that does not involve randomization. 
Consider a discrete-time optimal control problem with $K$ stochastic constraints, where the objective and constraints are on the expected cost over a finite horizon, $\{1 \cdots T\}$.
Let $u := \{u_1, u_2, \cdots, u_T\} \in \mathcal{U}^T$ be the control sequence, $x := \{x_1, x_2, \cdots, x_{T}\} \in \mathcal{X}^T$ be the state sequence, where $\mathcal{U}$ and $\mathcal{X}$ are the feasible control set and the state space, respectively. 
We denote by $w := \{w_1, w_2, \cdots, w_{T}\}$ the sequence of exogenous disturbance, which follows a known probability distribution. 
The system has a dynamics represented by $ g(x, u, w) = 0$. 
We consider a close-loop control, where the feedback law at $k$-th time step is given by a deterministic control policy $\mu_k: \mathcal{X} \mapsto \mathcal{U}$. Let $\mathcal{M}$ be the set of deterministic control policy that we consider. Hence, we seek for an optimal sequence of deterministic control policy, $\mu := \{\mu_1,\mu_2, \cdots, \mu_T\} \in \mathcal{M}^T$.
We define $K+1$ cost functions, $f_i : \mathcal{X}^T \times \mathcal{U}^T \rightarrow \mathbb{R}$, for $i = 0 \cdots K$. 
With a slight abuse of notation, we denote the close-loop cost and dynamics by $f_i(x, \mu)$ and $ g(x, \mu, w) = 0$, respectively.

The problem is to minimize the expectation of $f_0$ while constraining the expectations of  $f_1 \cdots f_K$ below $V_1 \cdots V_k$.
The minimized expected cost is denoted by $c_{\rm P}^\star$.

\vspace*{2mm}
\noindent {\bf PSOC (Pure-strategy Stochastic Optimal Control}
\begin{align}
c_{\rm P}^\star  = \min_{\substack{\mu \in \mathcal{M}^T \\ g(x, \mu, w) = 0}} & \ \mathbb{E}\left[ f_0(x, \mu) \right] \label{eq:objPure} \\
s.t. & \ \mathbb{E}\left[ f_i(x, \mu) \right] \le V_i, \quad i = 1 \cdots K. \label{eq:ccPure}
\end{align}
A notable example of constraints in the form of (\ref{eq:ccPure}) is a chance constraint. 
Let $\mathcal{X}_{\rm F} \subset \mathcal{X}$ be the set of feasible states.
A chance constraint imposes a bound on the probability that the state stays within $\mathcal{X}_{\rm F}$ over the planning horizon:
\begin{equation}
\Pr\left[ x \in \mathcal{X}_{\rm F}^T \right] \ge 1-V. \label{eq:cc}
\end{equation}
This constraint is posed in the form of (\ref{eq:ccPure}) by 
\[
f_i(x, \mu) := \left\{ 
\begin{array}{l}
0 \qquad (x \in \mathcal{X}_{\rm F}^T) \\ 
1 \qquad ({\rm Otherwise})
\end{array}
\right. .
\]
\todo{Any other example worth mentioning here?}

The problem is reduced to an open-loop control problem (i.e., optimization of control sequence) if $\mathcal{M}$ is limited to constant functions. Typical feedback MPCs limits $\mathcal{M}$ to linear feedback laws, $u = Kx$, where $K$ is optimized. MDP usually considers all the possible mappings with finite $\mathcal{X}$ and $\mathcal{U}$. 
Discussion in Section 2 poses no assumptions on $\mathcal{X}$, $\mathcal{U}$, and  $\mathcal{M}$ (except for standard assumptions such as \todo{what?}), hence it can be applied to a variety of problems ranging from stochastic MPC with continuous state and control to MDP with finite state and control. Then, Sections 3 and 4 discusses more specialized cases.

We next define the mixed-strategy problem, in which one of $N+1$ policy sequences is chosen at the beginning. 
Following the convention in constrained MDP, we call such a randomization the $N$-\textit{randomization} in this paper. 
Consider mixing $N+1$ policy sequences, $\mu^1 \cdots \mu^{N+1}$. Let $0 \le p^i \le 1$ be the probability that $\mu^i$ is chosen. 
The mixed strategy problem is to optimize $\mu^1 \cdots \mu^{N+1}$ as well as $p^1 \cdots p^{N+1}$ to minimize the expected cost.
The minimized expected cost is denoted by $c_{\rm M}^{\star N}$.

\vspace*{2mm}
\noindent {\bf MSOC$^N$ (Mixed-strategy Stochastic Optimal Control)}
\begin{align}
c_{\rm M}^{\star N} = \min_{\substack{\mu^1 \cdots \mu^{N+1} \in \mathcal{M}^T \\ \sum_{j=1}^N p^j = 1,\  p^j\ge0 \\ g(x, \mu, w) = 0}} & \ \sum_{j = 1}^{N+1} p^j \mathbb{E}\left[ f_0(x, \mu^j) \right] \label{eq:objMix} \\
s.t. & \ \sum_{j = 1}^{N+1} p^j \mathbb{E}\left[ f_i(x, \mu^j) \right] \le V_i \label{eq:ccMix} \\
& \qquad i = 1 \cdots K \nonumber
\end{align}

\subsection{Sufficient Degree of Randomization}
\label{sec:degree_of_randomization}
Before solving {MSOC$^N$}, we have to determine $N$.
In other words, we have to know what is the sufficient number of control sequence to be mixed.
We will show in this subsection that $K$-randomization ($N=K$) is sufficient in order to minimize $c_{\rm M}^{\star N}$.
In other words, for a problem with $K$ stochastic constraints, at most $K+1$ control sequences need to be mixed to form an optimal solution.
The formal statement is given in Theorem 1 later in this subsection. 
But we first need a few preparations.

Let $c = (c_0, c_1, \cdots c_K)$ where $c_i$ is the $i$-th cost value:
\[
c_i = \mathbb{E}\left[ f_i(x, \mu) \right].
\]

We denote by $\setF \subset \mathbb{R}^{K+1}$ the feasible set of the costs of the original problem, that is,
\begin{align}
\setF := \left\{ c \ | \ \mu \in \mathcal{M} \wedge g(x, \mu, w) = 0 \right\}.\label{eq:setF}
\end{align}
We assume that $\setF$ is a closed set, which typically holds when $\U$ and $\mathcal{X}$ are closed sets.
With $\setF$, {PSOC} (\ref{eq:objPure}), (\ref{eq:ccPure}) can be written in a simpler form as follows

\vspace*{2mm}
\noindent {\bf PSOC':}
\begin{align}
c_{\rm P}^\star  = \min_{c \in \setF} & \quad c_0 \label{eq:obj2} \\
s.t. & \quad  c_i \le V_i, \quad i = 1 \cdots K. \label{eq:cc2}
\end{align}

Likewise, {MSOC$^N$} is equivalent to:

\vspace*{2mm}
\noindent {\bf MSOC$^N$':}
\begin{align}
\min_{\substack{c^1 \cdots c^{N+1} \in \setF \\ \sum_{j=1}^{N+1} p^j = 1,\  p^j\ge0}} & \quad \sum_{j=1}^{N+1} p^j c^j_0 \label{eq:objMix3} \\
s.t. & \quad \sum_{j=1}^{N+1} p^j c^j_i \le V_i, \label{eq:ccMix3}
\end{align} 
where $c^j$ is the $j$-th cost vector and $c^j_i$ is its $i$-th component.	
We note that, when actually solving the problem, we do \textit{not} explicitly compute $\setF$. We introduce it for the ease of understanding.

For later convenience, we will derive another equivalent form to {MSOC$^N$}. Let 
\begin{align}
\setF_{\rm M}^N := \left\{ \sum_{j = 1}^{N+1} p^j c^j \ | \ c^j \in \setF, 0\le p^j, \sum_{j=1}^{N+1} p^j = 1 \right\}.\label{eq:setFM}
\end{align}
Observe that {MSOC$^N$'} is equivalent to:

\vspace*{2mm}
\noindent {\bf MSOC$^N$'':}
\begin{align}
c_{\rm M}^{\star N} = \min_{c \in \setF_{\rm M}^N} & \quad c_0 \label{eq:objMix2} \\
s.t. & \quad  c_i \le V_i, \quad i = 1 \cdots K. \label{eq:ccMix2}
\end{align} 

Let $\mathbb{N}$ be the set of positive integers. The following theorem holds:
\begin{theorem}
\begin{equation}
K \in \arg \min_{N \in \mathbb{N}} c_{\rm M}^{\star N}.
\end{equation}
\end{theorem}
\begin{proof}
From the definition of $\setF_{\rm M}^N$, it is obvious that 
\[
\setF_{\rm M}^i \subseteq \setF_{\rm M}^{i+1}, \quad \forall i \in \mathbb{N}.
\]
Therefore,
\begin{equation}
c_{\rm M}^{\star i} \ge c_{\rm M}^{\star i+1}, \quad \forall i \in \mathbb{N}. \label{eq:pf1}
\end{equation}
Also, since $\setF \subset \mathbb{R}^{K+1}$, it follows from Carath\/{e}odory's Theorem that
\[
\setF_{\rm M}^i = \conv(\setF), \quad \forall i \in \mathbb{N}, i \ge K,
\]
where $\conv(\cdot)$ is the convex hull of a set. Therefore, 
\begin{equation}
c_{\rm M}^{\star i} = c_{\rm M}^{\star i+1},\quad  \forall i \in \mathbb{N}, i \ge K.\label{eq:pf2}
\end{equation}
The theorem follows from (\ref{eq:pf1}) and (\ref{eq:pf2}).
\end{proof}

\MEC{ adding the Slater condition assumption.}
Theorem 1 means that we only need to consider $K$-randomization in order to minimize the expected cost.
In the remainder of this paper we only consider  {MSOC$^K$}, which we simply denote by {MSOC}. Its dual problem {DMSOC$^K$} will play an important role in the analysis further in the paper. Therefore, we will further assume that the Slater condition is satisfied for the {MSOC} problem, i.e., there exists a feasible point in the relative interior of conv$(\mathcal{F})$. The Slater condition guarantees a zero-duality gap. In general, it is easy to check and is rarely violated in practical problems.

We also use the following simplified notation:
\[
c_{\rm M}^{\star} := c_{\rm M}^{\star K}.
\]

\subsection{Cost Reduction by Randomization}
\label{sec:cost_reduction}

We next discuss under what condition the mixed strategy control can outperform pure strategies, and if it does, by how much.

\begin{lemma}\label{lem:dominance}  
The optimal mixed strategy control is at least as good as the optimal pure strategy control, that is:
\[
c_{\rm M}^{\star} \le c_{\rm P}^{\star}
\]
\end{lemma}
\begin{proof}
If follows from the following:
\[
\setF = \setF_{\rm M}^0 \subseteq \setF_{\rm M}^K.
\]
\end{proof}
This result is obvious because a pure strategy control can be viewed as a mixed strategy control that always assigns the probability of one to a single control sequence. 
The next question then is under what condition mixed strategies strictly dominate pure strategies.

\begin{lemma}\label{lem:strictDominance}
The necessary condition for 
\[
c_{\rm M}^{\star} < c_{\rm P}^{\star}
\]
is that $\setF$ is a non-convex set.
\end{lemma}
\begin{proof}
We prove the contraposition.
If $\setF$ is a convex set, then
\[
\setF = \conv(\setF) = \setF_{\rm M}^K.
\]
Hence,
\[
c_{\rm M}^{\star} = c_{\rm P}^{\star}.
\]
\end{proof}
The non-convexity of $\setF$ is not a sufficient condition for the strict dominance because the optimal solution to {MSOC$^K$} could be in $\setF$.
Also note that the convexity of $\setF$ is implied by the convexity of  {PSOC}, but not vice versa.

Figure \ref{fig:duality} provides a graphical interpretation of the above Lemmas in the case of $K=1$. 
The set painted in solid blue is $\setF$. 
Among  $\setF$, the areas to the left of the vertical line at $c_1 = V_1$ satisfies the constraint. 
Hence, the optimal solution to {PSOC} is located at the intersection of the vertical line and the lower edge of $\setF$, called the \textit{minimum common point} and shown in the green dot in Figure \ref{fig:duality}. 
Likewise, the optimal solution to {MSOC$^K$} is the minimum common point of $\conv(\setF)$ and the vertical line.
The dominance of mixed strategy (Lemma \ref{lem:dominance}), as well as the necessary condition for strict dominance (Lemma \ref{lem:strictDominance}), is graphically obvious from Figure \ref{fig:duality}.

\begin{figure}[tb]
  \begin{center}
    \includegraphics[width=\columnwidth]{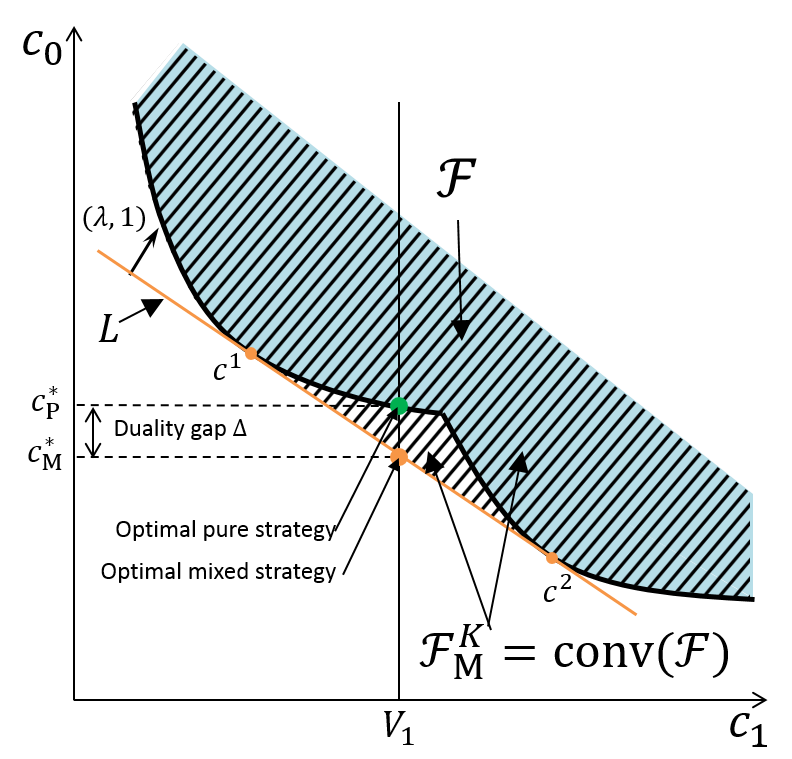}
  \end{center}
  \caption{
  \todo{Replace $\lambda^\star$ with $\lambda_1$. Add $L$. Add $\Delta$. add $c_1$, $c_2$.} 
  A graphical interpretation of Lemmas \ref{lem:dominance} and \ref{lem:strictDominance} and Theorem XX through the MC/MC framework \cite{Bertsekas_convex_optimization_theory}. Mixed strategy improves cost by extending the feasible search space through convexification.} 
  \label{fig:duality}
\end{figure}

What follows next is the discussion on by how much the mixed strategy can improve the expected cost, which requires some preparations.
The following is the dual optimization problem of {PSOC'}, where the dual optimal cost is denoted by $p_{\rm P}^\star$:

\vspace*{2mm}
\noindent {\bf DPSOC (Dual of PSOC)}
\begin{align}
q_{\rm P}^\star = \max_{\lambda \ge 0} \min_{c \in \setF} & \quad c_0 + \lambda(c_{1:K} - V), \label{eq:dual_orig} 
\end{align}
where $\lambda = [\lambda_1, \cdots \lambda_K]$ is the dual variables, $c_{1:K} = [c_1, \cdots c_K]^T$, and $V = [V_1, \cdots V_K]^T$ ($\cdot^T$ is the matrix transpose). 
From a standard result in optimization, $q_{\rm P}^\star \le c_{\rm P}^\star$.
 \todo{need citation?}.
The duality gap is denoted by $\Delta$, that is,
\[
\Delta = c_{\rm P}^\star - q_{\rm P}^\star.
\]

The dual optimization of {MSOC$^K$'} is given as follows:

\vspace*{2mm}
\noindent {\bf DMSOC$^K$ (Dual of MSOC$^K$')}
\begin{align}
q_{\rm M}^\star = \max_{\lambda \ge 0} \min_{c \in \conv(\setF)} & \quad c_0 + \lambda(c_{1:K} - V). \label{dual_mix}
\end{align}
Since $\conv(\setF)$ is convex, there is no duality gap, hence $q_{\rm M}^\star = c_{\rm M}^\star$.

It turns out that the improvement in expected cost by the optimal mixed strategy is equal to the duality gap of {PSOC}.
\begin{theorem}\label{th:gap}
\[
c_{\rm M}^\star = c_{\rm P}^\star - \Delta.
\]
\end{theorem}
\begin{proof}
Since $\setF \subseteq \conv(\setF)$,
\begin{equation}
\min_{c \in \conv(\setF)} \ c_0 + \lambda c_{1:K} \le \min_{c \in \setF} \ c_0 + \lambda c_{1:K}. \label{eq:proof1}
\end{equation}
The left hand side of the above is equivalent to 
\[
\min_{\substack{
c^1 \cdots c^{K+1} \in \setF \\ \sum_{i=1}^{K+1} p^i = 1,\  p^i\ge0}} \ \sum_{i=1}^{K+1} p^i (c^i_0 + \lambda c^i_{1:K}).
\]
A convex combination of a set of scalers cannot be less than the minimum of the set, i.e., for $\sum_{i=1}^n p_i = 1,  p_i\ge0, c^i \in \setF$, 
\[
\sum_{i=1}^{K+1} p^i (c^i_0 + \lambda c^i_{1:K}) \ge \min_{c \in \setF} \ c_0 + \lambda c_{1:K}.  
\]
Hence,
\begin{equation}
\min_{c \in \conv(\setF)} \ c_0 + \lambda c_{1:K} \ge \min_{c \in \setF} \ c_0 + \lambda c_{1:K}. \label{eq:proof2}
\end{equation}
From (\ref{eq:proof1}) and (\ref{eq:proof2}), it follows that
\[
\min_{c \in \conv(\setF)} \ c_0 + \lambda c_{1:K} = \min_{c \in \setF} \ c_0 + \lambda c_{1:K}. 
\]
Since $V$ is a constant, it follows from the above that $q_{\rm M}^\star = q_{\rm P}^\star$. Therefore, 
\begin{equation}
c_{\rm M}^\star = q_{\rm M}^\star = q_{\rm P}^\star = c_{\rm P}^\star - \Delta.\label{eq:hoge1}
\end{equation}
\end{proof}

The graphical interpretation of Theorem \ref{th:gap} is given by Figure \ref{fig:duality}.
Consider a hyperplane $L$ that contain $\setF$ in their upper closed halfspace and intersects with $\setF$, that is, $\setF \cap L$ is nonempty. Let the normal vector of $L$ be $[\lambda_1, \cdots \lambda_K, 1]$.
The points in $\setF \cap L$ correspond to the optimal solutions to the inner optimization problem of (\ref{eq:dual_orig}) given $\lambda$.
The value of $c_0$ at the crossing point between $L$ and $c_1 = V$ is the dual objective value.
Therefore, the optimal dual solution to {DPSOC} corresponds to the line that has the \textit{maximum crossing point} of $c_1 = V$ \cite{Bertsekas_convex_optimization_theory}.
Observe that the maximum crossing point for {DPSOC}, shown as the orange point in Figure \ref{fig:duality}, is the same for the minimum common point (i.e., the primal optimal solution) for {MSOC$^K$}.
Therefore the reduction in expected cost brought by mixed strategy is equivalent to the duality gap in {PSOC}.

\subsection{Solution approach}
\label{sec:solution}

A naive approach to solve {MSOC$^K$} is simply to solve (\ref{eq:objMix})-(\ref{eq:ccMix}).
However, the multiplication of $p^j$ increases the problem complexity (e.g., linear v.s. bilinear), making it difficult to solve.
Instead, in this paper, we present an efficient approach to solve {MSOC$^K$} by solving the dual of {PSOC}.
This approach is built on the fact revealed in the proof of Theorem \ref{th:gap} that the optimal dual solutions to {MSOC$^K$} and {PSOC} are the same.

Let $\lambda^\star$ be the optimal solution to {DPSOC}, (\ref{eq:dual_orig}).
Let $C(\lambda)$ be the set of all the optimal solutions to the inner optimization problem of {DPSOC}, that is,
\begin{equation}
C(\lambda) = \arg \min_{c \in \setF}  \quad c_0 + \lambda(c_{1:K} - V). \label{eq:dual_inner}
\end{equation}
For example, in case of Figure \ref{fig:duality}, $C(\lambda) = \{c^1, c^2\}.$ 
\todo{Fix the figure.}

\begin{theorem}\label{th3}
The necessary and sufficient condition for $(c^1 \cdots c^{K+1}, p^1 \cdots p^{K+1})$ to be an optimal solution to {MSOC$^K$'}, (\ref{eq:objMix3})-(\ref{eq:ccMix3}), is that there exist $\lambda = [\lambda_1 \cdots \lambda_K], \lambda_i \ge 0$, \MEC{ why not add this constraint into the equations?} that satisfy the followings:
\begin{align}
a) & \ c^i \in C(\lambda)  \ \vee \ p^i = 0, \ \forall i = 1 \cdots K+1 \nonumber\\
&\text{\MEC{ not sure if $\vee$ is the right symbol to use.}}\nonumber\\
b) & \ \lambda \left( \sum_{i=1}^{K+1} p^i c^i_{1:K} - V \right) = 0, \nonumber\\
c) & \ \sum_{i=1}^{K+1} p^i = 1,\nonumber \\
d) & \ p^i \ge 0 \ \forall i = 1 \cdots K+1, \nonumber \\
e) & \ \sum_{i=1}^{K+1} p^i c^i_{1:K} \le V , and \nonumber \\
f) & \ c^i \in \setF,  \ \forall i = 1 \cdots K+1.  \label{eq:th3}
\end{align}
\end{theorem}
\begin{proof} \ \\
\textit{Sufficiency}: 
It follows from d), e), and f) that $(c^1 \cdots c^{K+1}, p^1 \cdots p^{K+1})$ satisfies all the constraints of {MSOC$^K$'}. 
\MEC{ not clear why the following argument holds.}
With regard to a), note that:
\[
c^i \in C(\lambda) \Longleftrightarrow c^i_0 + \lambda (c^i_{1:K} - V) = q_{\rm P}^\star.
\]
It follows from a), b), and c) that
\begin{align}
   & \sum_{i=1}^{N+1} p^i c^i_0 \nonumber \\
= & \sum_{i=1}^{N+1} p^i c^i_0 + \lambda \left( \sum_{i=1}^{K+1} p^i c^i_{1:K} - V \right)  \nonumber \\
= & \sum_{i=1}^{N+1} p^i \{c^i_0 + \lambda (c^i_{1:K} - V)\} = q_{\rm P}^\star = c_{\rm M}^\star. \label{eq:th3-2}
\end{align}
Since we know that the minimum objective value of {MSOC$^K$'} is $c_{\rm M}^\star$, $(c^1 \cdots c^{K+1}, p^1 \cdots p^{K+1})$ is an optimal solution to {MSOC$^K$'}.

\textit{Necessity}: We prove the contraposition. Note that b)-f) are part of the KKT conditions [\MEC{ I think KKT conditions are for unconstrained optimization, but since $c$ must be in the set $\mathcal{F}$, then we need to use the more general theorem.}] for {MSOC$^K$'}.
Therefore, if any of b)-f) does not hold, $(c^1 \cdots c^{K+1}, p^1 \cdots p^{K+1})$ is not an optimal solution.  
Next, assume that only a) does not hold, that is,
\[
c^i_0 + \lambda (c^i_{1:K} - V) = q_{\rm P}^\star \ \wedge p^i > 0, \ \forall i = 1 \cdots K+1.
\]
Using ({eq:th3-2}), we have $\sum_{i=1}^{N+1} p^i c^i_0 > c_{\rm M}^\star$. Therefore $(c^1 \cdots c^{K+1}, p^1 \cdots p^{K+1})$ is not an optimal solution to {MSOC$^K$'}.
\end{proof}

\begin{remark}
Theorem \ref{th3}  can be immediately adapted to the original {MSOC$^K$}, i.e., (\ref{eq:objMix})-(\ref{eq:ccMix}). Let
\begin{equation}
U(\lambda) = \arg \min_{\substack{\mu \in \mathcal{M}^T \\ g(x, \mu, w) = 0}}  \ \mathbb{E}\left[ f_0(x, \mu) \right] - \lambda \left( \mathbb{E}\left[ f_{i:K}(x, \mu) \right] - V \right).\label{eq:inner}
\end{equation}
Then a) is replaced by the following condition: \\
a') \ $\mu^i \in U(\lambda)  \ \vee \ p^i = 0, \ \forall i = 1 \cdots K+1$.
\end{remark}

The uniqueness of Theorem \ref{th3} is in a). 
It means that the $K$ candidate control sequences, from which the controller choose randomly, can be obtained by solving the dual of the pure-strategy problem. 
More specifically, {MSOC$^K$} can be solved in the following process:
\begin{enumerate}
\item Solve {DPSOC} and obtain the optimal dual solution, $\lambda^\star$
\item If $\lambda^\star = 0$, optimal solutions to  {PSOC} are also optimal for {MSOC$^K$} because (\ref{eq:dual_inner}) reduces to {PSOC} with $\lambda = 0$.

\item If $\lambda^\star > 0$, \MEC{ $\lambda$ is a vector, so we need to replace by $\lambda^\star \neq 0$.}
\begin{itemize}
\item Solve (\ref{eq:dual_inner}) to obtain $C(\lambda^\star)$
\item Find $c^{i\star} \in C(\lambda^\star)$ and $p^{i\star} \ge 0$ such that  $\sum_{i=1}^{K+1} p^{i\star} c^{i\star}_{1:K} = V$ and $\sum_{i=1}^{K+1} p^{i\star} = 1$. \MEC{ This claim makes perfect sense geometrically, but is there a theoretical argument for the existence of such $c^{i\star} \in C(\lambda^\star)$? Also from the computation standpoint, $C(\lambda^\star)$ can be non-convex, so finding $c^{i\star}$ that satisfy the equations can be a difficult problem to solve.}
\end{itemize}
\end{enumerate}
The concrete solution method of {DPSOC} depends on problems.
In general it can be solved by a general convex optimization method such as subgradient method.
The multidimensional bisection method \cite{multidimentional_bisection} can solve it more efficiently if applicable.
More efficient and specialized solution approach would be available for special cases of {MSOC$^K$}.
However, such specialized solution approaches are out of the scope of this paper, except for the one that is discussed in the following subsection.


\subsection{Efficient Solution \MEC{ I'll be careful in using the word ``efficient'', in computer science it means a polynomial-time algorithm, so we need to do a further complexity study (for example, MILP cannot be efficiciently solved). } for $K=1$}\label{sec:root_finding}
PSOC with $K=1$ (i.e., there is only one stochastic constraint) has important applications, most notably the problems with a joint chance-constraint, which imposes the upper bound on the probability of violating \textit{any} constraints at \textit{any} time steps during the planning horizon. 
A joint chance-constraint has a practical importance since it provides the operator of a system an intuitive way to specify the acceptable level of risk of an entire plan. 
For example, in the Mars Exploration Rovers (MER) mission, ground operators made decisions on trajectory correction maneuver before atmospheric entry with a lower bound on the probability of successful landing (the thresholds for Spirit and Opportunity rovers were 91\% and 96\%, respectively) \cite{MarsLandingAnalysis2004}.

When $K$ = 1, DPSOC can be solved very efficiently by a root finding method. 
Furthermore, $\mu^1, \mu^2 \in U(\lambda^\star)$ and $c^1, c^2 \in C(\lambda^\star)$ are obtained as by-products of root finding. 
This involves evaluating the dual objective function repeatedly by solving (\ref{eq:inner}) with varying $\lambda \ge 0$.
The convergence is very fast; some of standard root finding algorithms, such as Brent's method, have a superlinear convergence rate.

\begin{figure}[tb]
  \begin{center}
    \includegraphics[width=0.8\columnwidth]{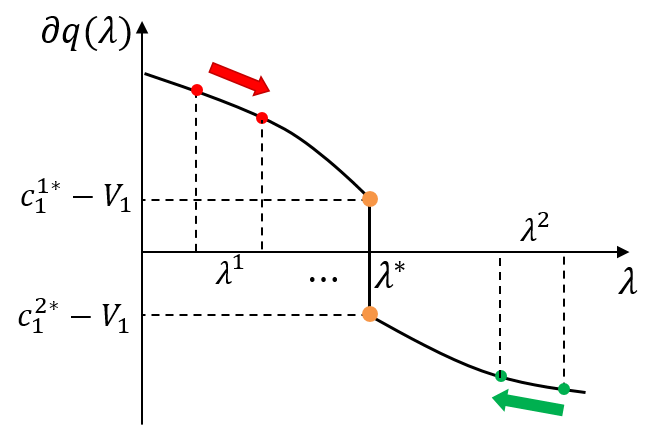}
  \end{center}
  \caption{The optimal dual solution is at the zero crossing of the subgradient of the dual objective, which is found by a root finding algorithm.} 
  \label{fig:subgradient}
  \begin{center}
    \includegraphics[width=\columnwidth]{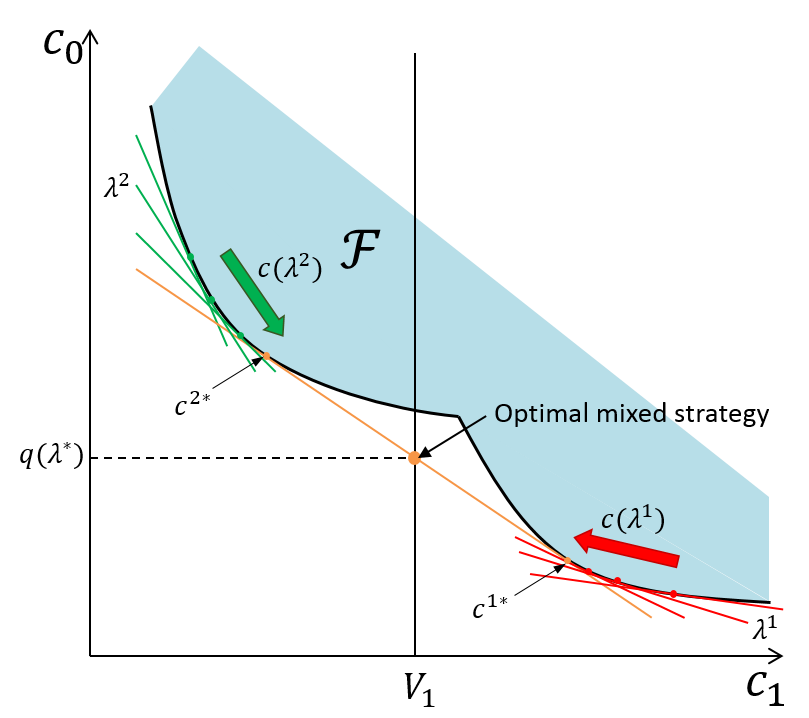}
  \end{center}
  \caption{
  \todo{Graphical interpretation of the proposed algorithm. A root-finding algorithm iteratively narrows the interval $[r_{\lambda_U^i}, r_{\lambda_L^i}]$. At convergence the optimal mixed strategy is at the intersection of $r = \Delta$ and the line that goes through  $(c_U^\star, r_U^\star)$ and $(c_L^\star, r_L^\star)$.}
  ~}
  \label{fig:algorithm}
\end{figure}

Let $q(\lambda)$ be the dual objective function of DPSOC, that is, 
\begin{equation}
q(\lambda) = \min_{c \in \setF} \quad c_0 + \lambda(c_1 - V_1).\label{eq:inner2}
\end{equation}
From a standard result of convex optimization theory, $q(\lambda)$ is a concave function \cite{Boyd_Textbook}, hence its subgradient $\partial q(\lambda)$ is monotonically decreasing, as shown in Figure \ref{fig:subgradient}.
Dual optimal solution, $\lambda^\star$, lies at the zero-crossing of $q(\lambda)$.
Also from a standard result of convex optimization theory is that: 
\[
c \in C(\lambda) \Rightarrow c_1 - V \in \partial q(\lambda).
\]
Therefore the dual optimization problem can be solved by finding a root of $c_1 - V$. Standard root finding algorithms can be used, such as bisection method and Brent's method \cite{Atkinson}. 

We assume that there is an algorithm that takes $\lambda$ and returns an optimal solution solution to
(\ref{eq:inner}), $\mu(\lambda) \in U(\lambda)$, as well as an optimal solution to (\ref{eq:inner2}), $c(\lambda) \in C(\lambda)$, which satisfies a) and f) of Theorem 3.
We denote by $c_i(\lambda)$ the $i$-th component of $c_i(\lambda)$.
The root finding algorithm is initiated with an interval, $[\lambda^1, \lambda^2]$, which includes $\lambda^\star$.
The interval is tightened iteratively until a certain terminal condition is met.
Through the iteration, $c(\lambda^1)$, $c(\lambda^2)$, $\mu(\lambda^1)$, and $\mu(\lambda^2)$ converge to $c^{1\star}$, $c^{2\star}$, $\mu^{1\star}$, and $\mu^{2\star}$, respectively, while $\lambda^1$ and $\lambda^2$ converge to $\lambda^\star$, as illustrated in Figure \ref{fig:subgradient}.
If $\lambda^\star > 0$, $p^{1\star}$ and $p^{2\star}$ that satisfies b), c), and e) in Theorem 3 are computed by solving the following:
\begin{align*}
& p^{1\star} c^{1\star}_1 + p^{2\star} c^{2\star}_1 = V_1 \\
& p^{1\star} + p^{2\star} = 1.
\end{align*}
The solution to the above also satisfies d) because $\lambda^\star \in [\lambda^1, \lambda^2]$ implies $c_1(\lambda^1) \ge V_1 $ and $c_1(\lambda^2) \le V_1 $.
Therefore, $(c^{1\star}, c^{2\star}, p^{1\star}, p^{2\star})$ satisfies a)-f) of Theorem 3, hence it is an optimal solution to MSOC$^1$.
If $\lambda^\star = 0$, an optimal solution is $c^{1\star} = c^{2\star} = c(\lambda^\star)$ and $p^{1\star}$ and $p^{2\star})$ can be any that satisfies c) and d).

The optimal mixed control is to execute $\mu^{1\star}$ with probability $p^{1\star}$ and $\mu^{2\star}$ with $p^{2\star}$.

\section{Deployment on Linear SMPC}
The proposed algorithm is demonstrated with an implementation on a linear SMPC with normally distributed disturbance and polygonal obstacles in the state space.
Since the problem is nonconvex, a mixed strategy may outperform pure strategies.
A practical challenge is that (\ref{eq:inner}) is nonlinear, nonconvex programming. The nonlinearity comes from the cumulative distribution function (CDF) that is used to evaluate the probability of constraint violation. 
Although an efficient solvers are available for a limited classes of nonconvex programming such as mixed integer linear programming (MILP) and mixed integer quadratic programming (MIQP), the problem does not fall under these classes. 

Repeatedly solving such a problem could result in a prohibitive cost.
Our approach is to approximate the CDF with a piecewise linear function and convert the problem into MILP.

\subsection{Formulation}
We assume a linear discrete-time dynamics with $x_k \in \R^n$ and $u_k \in \U \subset \R^m$:
\[
x_{k+1} = Ax_k + Bu_k + w_k,
\]
where $w_k$ is a normally distributed zero-mean disturbance with the covariance of $\Sigma_w$. 
$\U$ is assumed to be a polytope, hence $\U = \{ u \in \R^m | P u \preceq q \}$, 
where $\preceq$ and $\succeq$ are the componentwise inequalities.
We assume that there are $M$ polytopic obstacles, whose interior is represented as:
\[
H_i x_k \succeq g_i, \quad  i = 1 \cdots M, \ k = 1 \cdots N.
\] 
A chance constraint in the form of (\ref{eq:cc}) is imposed to limit the probability of the violation of the obstacles is limited to $V$.
The cost function is the total $L^1$ norm of $u_k$ over the horizon, that is, $f_0(x, u) = \sum_{k=1}^N |u_k|_1$.
Since this cost function is deterministic, $\mathbb{E}\left[ f_0(x, u) \right] = f_0(x, u)$.
The PSOC is given as follows:
\begin{align*}
\min & \ \sum_{k=1}^N |u_k|_1\\
{\rm s.t.} & \ \Pr\left[ \bigvee_{i=1}^{M} \bigvee_{k=1}^{N} H_i x_k \succeq g_i \right] \le V  \\
& \  x_{k+1} = Ax_k + Bu_k + w_k, \ k = 1 \cdots N \\
& \ P u_k \preceq q, \ k = 1 \cdots N,
\end{align*}
where $\bigvee$ is the logical disjunction.
The inner optimization problem of the dual optimization, (\ref{eq:inner}), is given as:
\begin{align*}
\min & \ \sum_{k=1}^N |u_k|_1 + \lambda \left(\Pr\left[ \bigvee_{i=1}^{M} \bigvee_{k=1}^{N} H_i x_k \succeq g_i \right] - V \right)\\
{\rm s.t.} & \  x_{k+1} = Ax_k + Bu_k + w_k, \ k = 1 \cdots N \\
& \ P u_k \preceq q, \ k = 1 \cdots N.
\end{align*}

\subsection{Conversion to MILP}\label{sec:MILP}

We use a few tricks and approximations to convert the above problem into MILP.
We note that the probability of constraint violation is always approximated conservatively (meaning that it is overestimated) so that a solution to the approximated problem is always a feasible solution to the original problem. 
First, by replacing absolute values with slack variables, the $L^1$ norm objective is equivalent to the following:
\[
\min |u|_1 = \min \sum_{d=1}^m v_d \quad {\rm s.t.} \ v_d \ge u_d,  v_d \ge -u_d,
\]
where $u_d$ is the $d$-th component of vector $u$. 
Second, the joint probability is decomposed by Boole's inequality, whose conservatism is trivial in most practical cases where the risk bound $V$ is very small (e.g., $< 0.01$) \cite{Hiro_AAAI08}:
\[
\Pr\left[ \bigvee_{i=1}^{M} \bigvee_{k=1}^{N} H_i x_k \succeq g \right] \sim \sum_{i=1}^{M} \sum_{k=1}^{N} \Pr \left[H_i x_k \succeq g \right].
\]
The componentwise inequality in the probability is decomposed using the risk selection approach \cite{Hiro_ACC10}, which is again a conservative approximation. Let $h_{ij}$ and  $g_{ij}$ be the $j$-th row of $H_i$ and $g_i$, and $R_i$ be the number of rows,
\begin{equation}
\min  \Pr \left[H_i x_k \succeq g_i \right] \sim \min \delta \quad {\rm s.t.} \ \bigvee_{j=1}^{R_i} \Pr \left[h_{ij} x_k \ge g_{ij} \right] \le \delta. \label{eq:hoge}
\end{equation}
The probability above is univariate, hence it can be easily evaluated by CDF:
\[
\Pr \left[h_{ij} x_k \ge g_{ij} \right] = F\left( \frac{h_{ij}\bar{x}_k - g_{ij}}{h_{ij}\Sigma_{x_k}h_{ij}^T} \right),
\] 
where $\bar{x}_k$ is the mean of $x_k$ and $F(\cdot)$ is the CDF of the standard normal distribution. The covariance matrix of $x_k$ is computed recursively by $\Sigma_{x_{k+1}} = A \Sigma_{x_k} A^T + \Sigma_w$.

We apply a piecewise linear approximation of the CDF. 
Since the CDF of the standard normal distribution $F(y)$ is convex at $y \le 0$, the piecewise linear approximation can be done \textit{without} introducing integer variables.
An underlying assumption is that the mean state $\bar{x}_k$ is always outside of obstacles, hence $h_{ij}\bar{x}_k - g_{ij} < 0.$ 
This assumption is implied by $V < 0.5$ because if the mean state is on a constraint boundary, the probability of violating the constraint is 0.5. In practical cases the users usually do not allow 50 \% of risk. 
Let $a_l y + b_l$ be the linear approximation of $F(y)$ at $y_l \le 0$, $l = 1 \cdots L$. 
The right hand side of (\ref{eq:hoge}) is approximated as follows:
\[
\min \delta \quad {\rm s.t.} \ \bigvee_{j=1}^{R_i} \bigwedge_{l=1}^L a_l\left( \frac{h_{ij}\bar{x}_k - g_{ij}}{h_{ij}\Sigma_{x_k}h_{ij}^T} \right) + b_l \le \delta.
\]

Finally, the disjunction is replaced by mixed-integer constraints using a standard trick called the big-M method \cite{BertsimasLP}. Letting $\mathcal{M}$ be a very large positive constant, the optimization problem formulated in the previous subsection is now converted to MILP as follows:

\begin{align*}
\min & \ \sum_{k=1}^N \sum_{i=1}^m v_{ki} + \lambda \left( \sum_{i=1}^M \sum_{k=1}^N \delta_{ik} - V \right) \\
{\rm s.t.} & \ v_{kd} \ge u_{kd}, \ v_{kd} \ge -u_{kd}, P u_k \preceq q\\
& \  \bar{x}_{k+1} = A \bar{x}_k + Bu_k + w_k\\
& \ a_l \left( \frac{h_{ij} \bar{x}_k - g_{ij}}{h_{ij}\Sigma_{x_k}h_{ij}^T} \right) + b_l \le \delta_{ik}+ \mathcal{M}z_{ij} \\
& \ \sum_{j = 1}^{R_i} z_{ij} \le {R_i}-1, \ \ z_{ij} \in \{0, 1\} \\ 
& \ k = 1 \cdots N, d = 1 \cdots m, \ i = 1 \cdots M,  \\ 
& \ j = 1 \cdots R_i, \ l = 1 \cdots L.
\end{align*}

\subsection{Simulation Results}

We performed simulations on a double integrator plant:
\begin{align*}
\small
&A = \left[\begin{array}{cccc}
1 & 0 & \Delta T & 0       \\ 
0 & 1 & 0        & \Delta T \\ 
0 & 0 & 1        & 0       \\ 
0 & 0 & 0        & 0
\end{array}\right], \
B = \left[\begin{array}{cc}
\frac{1}{2}\Delta T^2 	& 0       \\ 
0        				& \frac{1}{2}\Delta T^2\\ 
\Delta T    			& 0       \\ 
0        				& \Delta T
\end{array}\right], \\
&\Sigma_w = \left[\begin{array}{cccc}
\sigma_w^2	& 0 			& 0 	& 0       \\ 
0 			& \sigma_w^2 	& 0     & 0 \\ 
0 			& 0 			& 0		& 0       \\ 
0 			& 0 			& 0     & 0
\end{array}\right], \quad \Delta T = 1, \ \sigma_w = 0.1.
\end{align*}

We first considered an illustrative example shown in Figure \ref{fig:sim}, where two obstacles were placed between which there was a narrow shortcut passage. Initial state was at $[ 0, 0, 0, 0]^T$ and the mean final state was constrained at $[ 10, 10, 0, 0]^T$. Horizon length was $N=15$, and finally the risk bound is $V = 0.01$. The simulation was performed on a machine with Intel Core i7-3612QM CUP clocked at 2.10 GHz and 8.00 GB RAM. The algorithm was implemented in MATLAB using YALMIP \cite{YALMIP} and MILP was solved by CPLEX. The bisection method was used for root finding. For comparison, the optimal pure strategy was also computed by using the same MILP approximation presented in Section \ref{sec:MILP}.

Figure \ref{fig:sim} shows the mixed and pure strategy solutions computed by the proposed algorithm. The mixed strategy consisted of the two control sequences shown in blue lines. The lower dual solution $\lambda_L$ corresponds to a risk-taking path that goes through the narrow passage, while the upper dual solution $\lambda_U$ results in a risk-averse path that go around the obstacles. \MEC{ Adding a table for the bisection method showing the changes (and convergence) to the two values of $\lambda$ can be useful.} The former took $r_L = 0.0278$ of risk and $c_L = 3.692$ of cost, while the latter took $r_U = 0.0021$ of risk and $c_U = 4.175$ of cost.
The mixed strategy chose between them with the probabilities of $0.306$ and $0.694$, resulting in the risk of exactly 0.01 and the expected cost of 4.027. On the other hand, the pure optimal strategy took the risk of exactly 0.01 and the cost of 4.175. Therefore this example validates our claim that mixed strategy can result in less expected cost than pure strategy in a nonconvex SMPC.
\begin{figure}[tb]
  \begin{center}
    \includegraphics[width=\columnwidth]{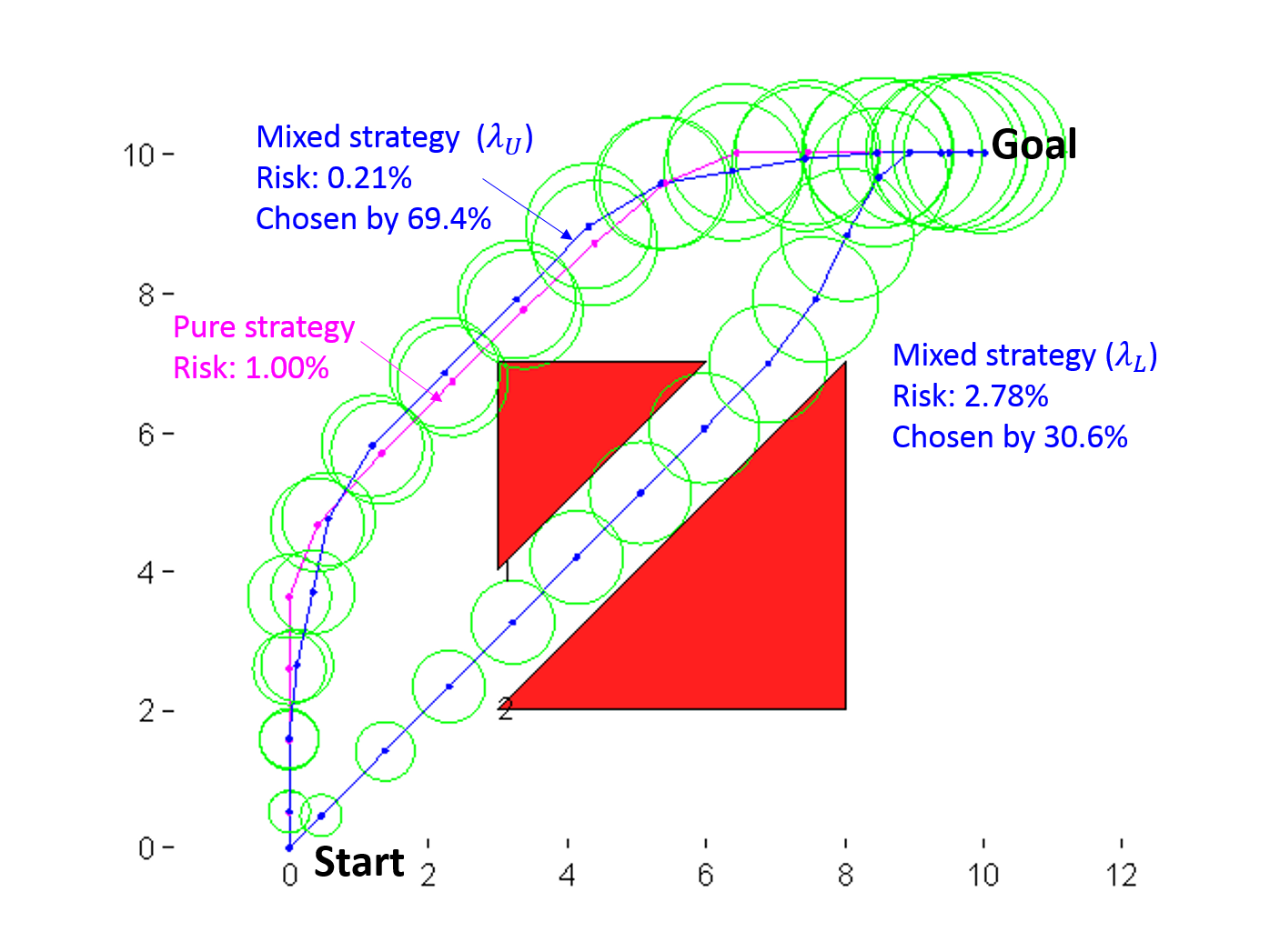}
  \end{center}
  \caption{Simulation result on an illustrative example. The optimal mixed strategy chooses between the two blue control sequences. The expected cost of the mixed control strategy is 4.027 while that of the optimal pure control strategy is 4.105. $c_L = 3.692, c_U=4.175$. The three-sigma ellipses are shown in green.} 
  \label{fig:sim}
\end{figure}

We next performed a Monte Carlo simulation in order to empirically validate our claim that the optimal solution to the mixed strategy problem is always as good as the optimal solution to the original (pure) problem. 
We randomly placed four square obstacles in a 2-D state space. The center of each square was sampled from a uniform distribution within $[-3, 3] \times [-3, 3]$. The size of each square was sampled from a uniform distribution in $[1, 3]$. The initial state was  $[ -5, 0, 0, 0]^T$, and the mean final state was constrained at $[ 5, 0, 0, 0]^T$. 

Figure \ref{fig:MC} shows the resulting cost of the optimal mixed and pure solutions to 200 randomized problems.
The average computation time was 59.8 sec.
There were 163 samples on the $45^\circ$ line in the plot, meaning that the cost of optimal mixed and pure solutions were identical in those samples. 
There were 37 samples below the $45^\circ$ line, meaning that the cost of optimal mixed solution was strictly less than the cost of the optimal pure solution. 
There was no sample above the $45^\circ$ line.
This result supports our claim that the optimal solution to the mixed strategy problem is always as good as the optimal solution to the original problem.
At least in this particular problem domain, mixed strategy outperforms pure strategy not very frequently, and as is seen in Figure \ref{fig:MC} the improvement is often marginal. 
It is certainly possible to engineer a problem that better highlights the advantage of mixed strategy, but that does not serve the objective of this paper.
The most important contributions of this paper are the theoretical finding that mixed strategy can outperform pure strategy in nonconvex SMPCs, as well as the algorithm to compute the optimal mixed strategy solutions.
The empirical results validate the theoretical finding and the algorithm.

\begin{figure}[tb]
  \begin{center}
    \includegraphics[width=\columnwidth]{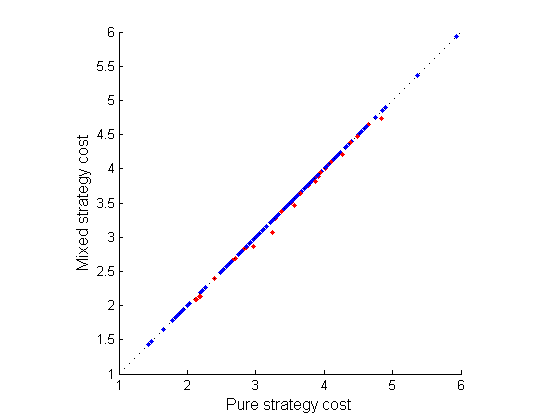}
  \end{center}
  \caption{Cost of the optimal pure and mixed solutions to 200 problems with randomly generated obstacles. Red dots are the samples where the cost of the optimal mixed solution is strictly less than the optimal pure solution. Blue dots are the samples where two costs are identical. No optimal pure solutions outperform the optimal mixed solutions. } 
  \label{fig:MC}
\end{figure}

\section{Deployment on Chance-constrained MDP}
The application of the proposed approach is not limited to SMPC.
In this section we present applications to finite-state MDPs with a chance constraint.

We consider a finite time steps, $k = 1 \cdots T$.
The state space and action space are finite and time-varying, denoted by $\mathcal{X}_k$ and $\mathcal{U}_k$. 
State and control sequence variables are represented as $x_k \in \mathcal{X}_k$ and $u_k \in \mathcal{U}_k$.
A control policy is a map $\mu_k: \mathcal{X}_k \rightarrow \mathcal{U}_k$.
The sequence of control policy is denoted by $\mu := [\mu_1, \cdots, \mu_T]$.
A mixed strategy finds multiple control policy sequences, $\mu^1, \mu^2, \cdots \mu^{K+1}$, and randomly choose one.
The control objective is to minimize the expected total cost, $\mathbb{E}\left[ \sum_{k=1}^T f_0(x_k, u_k) \right]$.
A set of failure states, $\mathcal{X}_k^F \subset \mathcal{X}_k$, is specified for each time step.
A joint chance constraint limits the probability that one of the failure states is visited at any time step:
\[
\Pr\left[ \bigvee_{k=1}^{T} x_k \in \mathcal{X}_k^F \right] \le V.
\]

Since $K=1$, the mixed-strategy problem can be solved by root finding, as in Section \ref{sec:root_finding}.
The inner optimization problem is solved through the chance-constrained dynamic programming\cite{Hiro_AURO2015}
\footnote{\todo{more precisely, CCDP is also solved by dual. the inner optimizaton of CCDP}}.
In the reminder of this section we present two applications: path planning and Mars Entry, Descent, and Landing (EDL).

\subsection{Application to Path Planning}
In this application, we assume $T=50$ and a two-dimensional state space in $[0, 100]^2$ discretized into a 100x100 grid.
Obstacles are placed as shown in Figure \ref{fig:path}.
A single integrator dynamics is assumed:
\begin{align*}
& x_{k+1} = x_{k} + u_{k} + w_{k} \\
& \| u_k \|_2 \le d_k, \quad w_k \sim \mathcal{N}(0, \sigma^2 I),
\end{align*}
where $u_{k}$ is a two dimensional vector specifying the increment in position, $w_{k}$ is a discritized, Gaussian-distributed noise, $d_k$ and $\sigma$ are constant parameters, $\mathcal{N}(0, \Sigma)$ is a zero-mean Gaussian distribution with the covariance matrix $\Sigma$, and $I$ is the two-dimensional identity matrix.
We set $d_k = 6$ and $\sigma = 1$.
The cost function is the expected length of the resulting path that connects the start and goal states. 
The risk bound is $V = 0.02$.

The optimal solution to MSOC consists of two control policy sequences, $\mu^1$ and $\mu^2$, which have expected path lengths of $130.8$ and $98.7$ while the risks of hitting obstacles being $0.64\%$ and $2.28\%$, respectively.
The nominal paths resulting from $\mu^1$ and $\mu^2$ (i.e, state sequence assuming when $w_{k} = 0$) are shown in Figure \ref{fig:path}. 
The two pure control strategies are chosen with probabilities of $17.0\%$ and $83.0\%$, respectively.
As a result, the mixed control strategy has a expected path length of $104.2$ while the risk of hitting obstacles is exactly $2.0\%$.
On the other hand, solving PSOC results in the same pure control strategy  as $\mu^1$, whose expected path length is $130.8$.
As expected, mixed strategy resulted in a less expected cost while respecting the stochastic constraint.

The solution time of MSOC was $20.52$ seconds while that of the PSOC was $20.38$ seconds\footnote{Simulations are conducted on a machine with the Intel(R) Xenon(R) X5690 CPU clocked at 3.47GHz and 96GB of RAM}.
The difference in computation time is small because 
solving PSOC also requires iterative dual optimization in this case\cite{Hiro_AURO2015}.

\begin{figure}[tb]
  \begin{center}
    \includegraphics[scale=0.5]{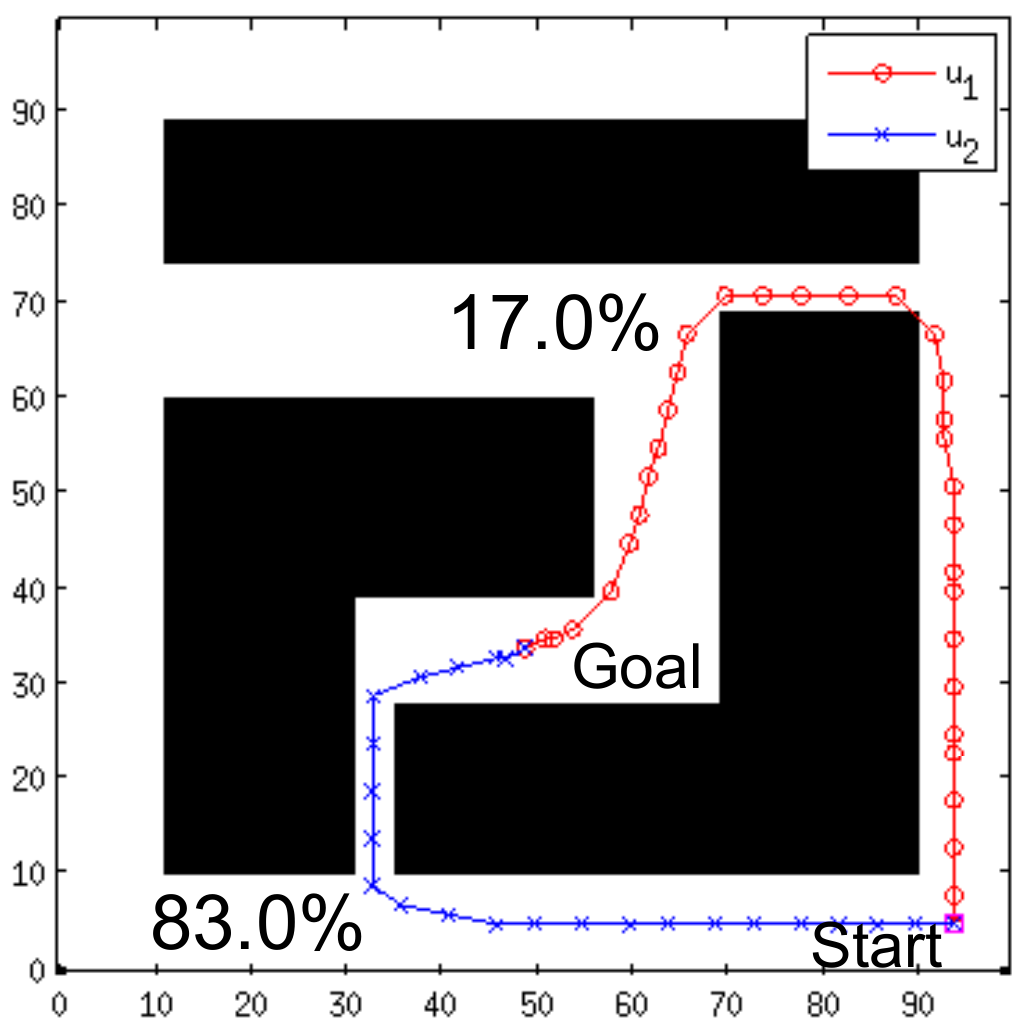}
  \end{center}
  \caption{Application of mixed-strategy stochastic control to a path planning problem. The optimal mixed control strategy chooses between the two paths, $u_1$, $u_2$, with probabilities of $17.0\%$ and $83.0\%$, respectively.} 
  \label{fig:path}
\end{figure}

\subsection{Application to Mars Entry, Descent, and Landing}
We next present an application to the planning of entry, descent, and landing (EDL) for future Mars missions\cite{Hiro_AURO2015}.
Mars EDL is subject to various source of uncertainties such as atmospheric
variability and imperfect aerodynamics model.
The resulting dispersions of the landing position typically spans over
tens of kilometers for a 99.9\% confidence
ellipse~\cite{MarsLandingAnalysis2004}. 
Given such a highly uncertain
nature of EDL, a target landing site must be carefully chosen in order
to limit the risk of landing on rocky or uneven terrain.
At the same time, it is equally important to land near science targets in order to minimize
the traverse distance after the landing. 

Future Mars lander/rover missions would aim to reduce the uncertainty by using several new active control technologies, consisting of the
following three stages: entry-phase targeting, powered-descent guidance
(PDG)~\cite{Acikmese07JGCD}, and hazard detection and avoidance
(HDA)~\cite{Johnson08AC_HDA}. 
Each control stage is capable of making
corrections to the predicted landing position by a certain distance,
but each stage is subject to execution errors, which deviates the
spacecraft away from the planned landing position. 

We pose this problem as an optimal sequential decision making under a persisting
uncertainty.
At the $k$th control stage, $x_k$ represents the projected landing location without further control.
By applying a control at the $k$th stage, the lander can correct the projected landing location to $u_k$, which must be within an ellipsoid centered around $x_k$. At the end of the $k$th control stage, the projected landing location $x_{k+1}$ deviates from $u_k$ due to a disturbance $w_k$, which is assumed to have a Gaussian distribution.
This EDL model is described as follows:
\begin{align*}
& x_{k+1} = u_{k} + w_{k} \\
& (u_k - x_k)^T D_k (u_k - x_k) \le d_k^2, \quad w_k \sim \mathcal{N}(0, \Sigma_k),
\end{align*}
where $D_k$ and $\Sigma_k$ are positive definite matrices, and $d_k$ is a scalar constant.
We use the same parameter settings as \cite{Hiro_CDC12_DP}.

We consider three control stages, i.e., $T=3$ and $x_3$ is the final landing location.
The state space $\mathcal{X}$ is a 2 km-by-2 km square, which is discretized at a one meter resolution.
As a result, the problem has four million states at each time step.
The control and the disturbance are also discretized at the same resolution.
The cost function is the expected distance to drive on surface to visit two science targets, shown in magenta squares in  Figure \ref{fig:result}, starting from the landing location.
The infeasible areas are specified using the data of HiRISE (High Resolution Imaging Science Experiment) camera on the Mars Reconnaissance Orbiter.
We use the real landscape of a site named ``East Margaritifer" on Mars.

Figure \ref{fig:result} show the simulation result with a risk bound $V = 0.1\%$.
The optimal solution to MSOC chooses between two control policy sequences, $\mu^1$ and $\mu^1$, with the probabilities of $84.9\%$ and $15.1\%$.
The probability of failure of the two control policy sequences are $0.016\%$ and $0.574\%$ while their costs being $645.49$ and $641.02$.
The resulting probability of failure of the mixed strategy is exactly $0.1\%$
The optimal solution to PSOC is equivalent to $\mu^1$.
Again, as expected, mixed strategy reduces the expected cost while respecting the stochastic constraint.

Note that the optimal pure control policy takes significantly less risk than the risk bound.
This is because there is no other solution that is within the risk bound and has less cost.
The mixed control strategy improves the cost by mixing this optimal pure control strategy with another control policy that has an excessive risk but a less cost.

\begin{figure}[tb]
\centering 
    \subfigure[Optimal pure control strategy]{\includegraphics[scale=.45]{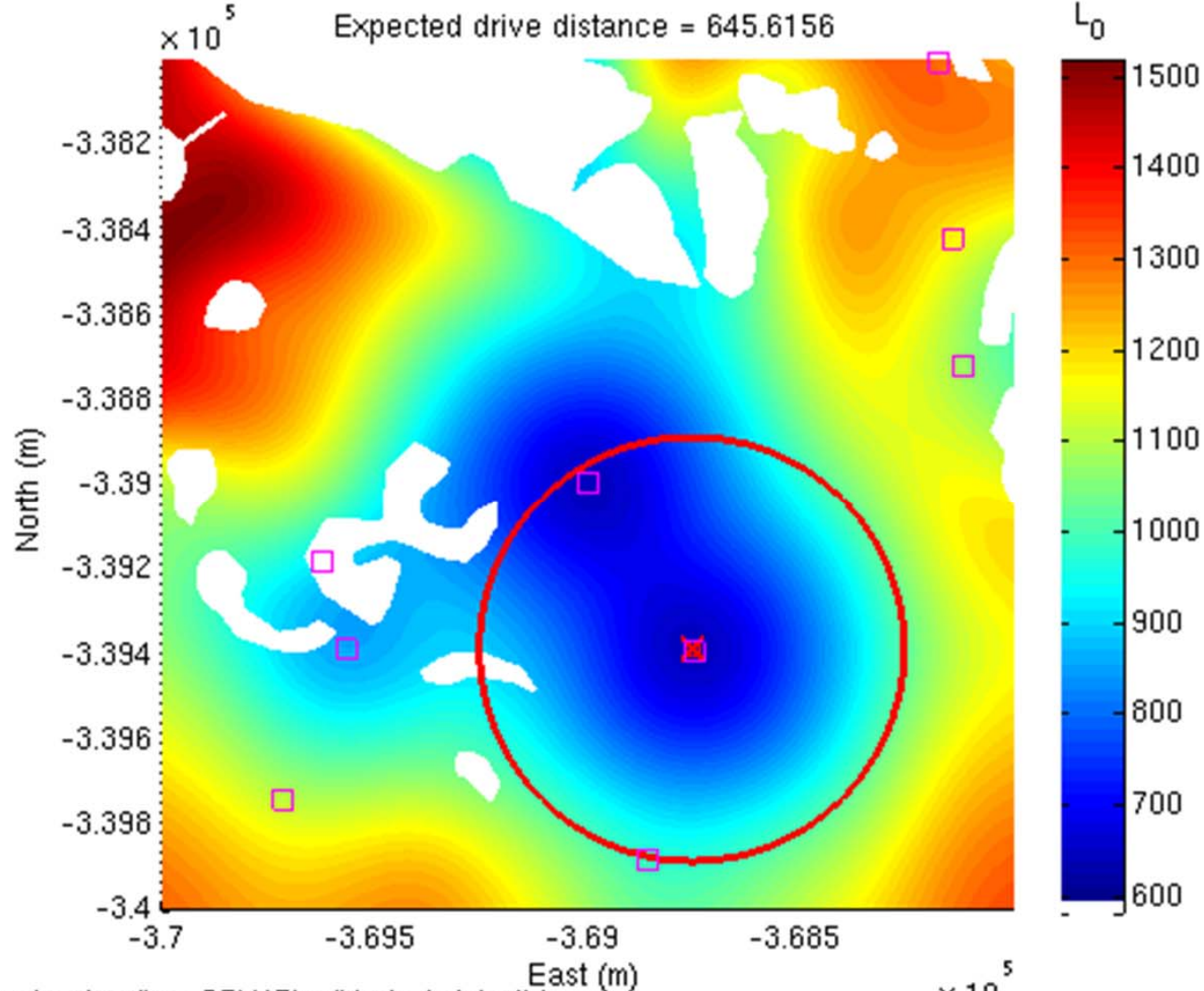}
   \label{fig:mars_1}}
    \subfigure[Optimal mixed control strategy]{\includegraphics[scale=.45]{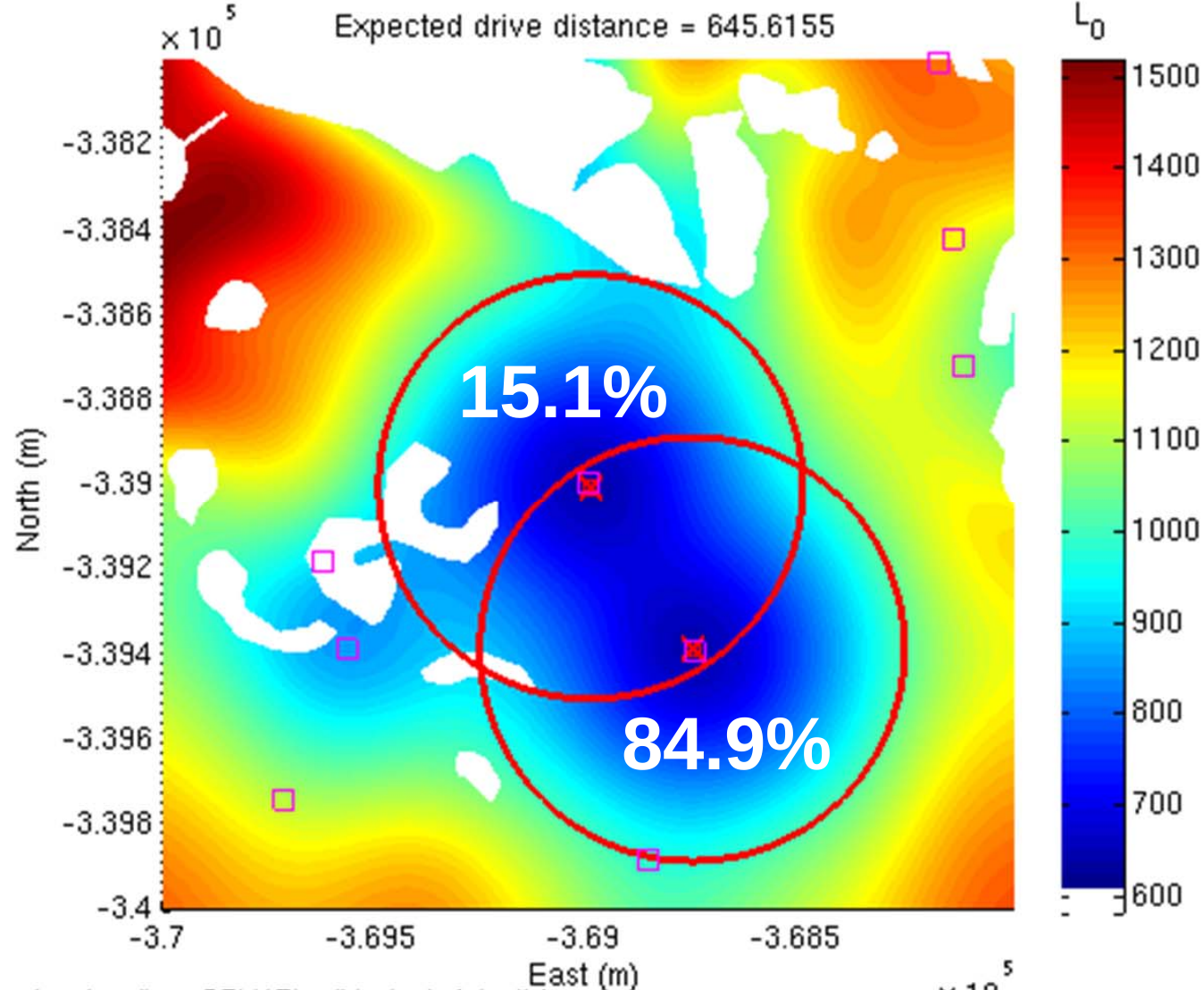}
   \label{fig:mars_2}}
  \caption{(a) The optimal pure control strategy and (b) the optimal mixed control strategy for the Mars EDL scenario with a risk bound $V = 0.01$. The red $\times$-marks are the optimal EDL target of the pure control strategies, while the red circles represents $3\sigma$ of the disturbance in the first stage, $w_o$. The mixed strategy chooses between the two pure control strategies with the probabilities of $15.1\%$ and $84.9\%$.}
  \label{fig:result}
\end{figure}

\todo{Move this discussion to intro?}
It may sound unrealistic to decide a landing site probabilistically.
However, consider a situation where there are 1,000 vehicles and we require 999 of them to land successfully while minimizing the total cost.
Then our result means that the optimal strategy is to send 849 of them to the first landing site and 151 of them to the other.
When having only one vehicle, the interpretation of this result varies with viewpoint.
For a person who knows the result of the coin flip in advance of the landing, the resulting action is no more mixed and hence it may violate the given chance constraint.
However, if the result of the coin flip is hidden from the observer, like Schr\"{o}dinger's cat in a box, then this mixed strategy results in the minimum expected cost while the probability of failure is still within the specified bound.

\section*{Conclusions}
We found that, in nonconvex SMPC, choosing control inputs randomly can result in a less expected cost than deterministically optimizing them.
We developed a solution method based on dual optimization and deployed it on a linear nonconvex SMPC problem, which was efficiently solved using an MILP approximation.
Finally, we validated our theoretical findings through simulations.

\section*{Acknowledgment}

The research described in this paper was in part carried out at the
Jet Propulsion Laboratory, California Institute of Technology,
under a contract with the National Aeronautics and Space
Administration. 
This research was supported by the Office of Naval Research, Science of Autonomy Program, under Contracts N00014-15-IP-00052 and N00014-15-1-2673.

\ifCLASSOPTIONcaptionsoff
  \newpage
\fi

\bibliographystyle{IEEEtran}
\bibliography{arxiv}

%


\end{document}